\documentclass[sigconf]{acmart}
\copyrightyear{2025}
\acmYear{2025}
\setcopyright{cc}
\setcctype{by-nc}
\acmConference[WWW '25]{Proceedings of the ACM Web Conference 2025}{April 28-May 2, 2025}{Sydney, NSW, Australia}
\acmBooktitle{Proceedings of the ACM Web Conference 2025 (WWW '25), April 28-May 2, 2025, Sydney, NSW, Australia}
\acmDOI{10.1145/3696410.3714853}
\acmISBN{979-8-4007-1274-6/25/04}

\title{InfoMAE: Pair-Efficient Cross-Modal Alignment for Multimodal Time-Series Sensing Signals}

\author{Tomoyoshi Kimura}
\email{tkimura4@illinois.edu}
\affiliation{%
  \institution{University of Illinois Urbana-Champaign}
  \city{Urbana}
  \state{IL}
  \country{USA}
}

\author{Xinlin Li}
\author{Osama Hanna}
\email{xinlinli@g.ucla.edu}
\email{ohanna@ucla.edu}
\affiliation{%
  \institution{University of California, Los Angeles}
  \city{Los Angeles}
  \state{CA}
  \country{USA}
}

\author{Yatong Chen}
\email{chenyatong@sjtu.edu.cn}
\affiliation{%
    \institution{Shanghai Jiao Tong University}
    \city{Shanghai}
    \country{China}
}

\author{Yizhuo Chen}
\author{Denizhan Kara}
\email{yizhuoc@illinois.edu}
\email{kara4@illinois.edu}
\affiliation{%
  \institution{University of Illinois Urbana-Champaign}
  \city{Urbana}
  \state{IL}
  \country{USA}
}

\author{Tianshi Wang}
\author{Jinyang Li}
\email{tianshi3@illinois.edu}
\email{jinyang7@illinois.edu}
\affiliation{%
  \institution{University of Illinois Urbana-Champaign}
  \city{Urbana}
  \country{USA}
}

\author{Xiaomin Ouyang}
\email{xmouyang@cse.ust.hk}
\affiliation{%
  \institution{Hong Kong University of Science and Technology}
  \city{Hong Kong SAR}
  \country{China}
}

\author{Shengzhong Liu}
\email{shengzhong@sjtu.edu.cn}
\affiliation{%
  \institution{Shanghai Jiao Tong University}
  \city{Shanghai}
  \country{China}
}

\author{Mani Srivastava}
\author{Suhas Diggavi}

\email{mbs@ee.ucla.edu}
\email{suhas@ee.ucla.edu}
\affiliation{%
  \institution{University of California, Los Angeles}
  \city{Los Angeles}
  \state{CA}
  \country{USA}
}

\author{Tarek Abdelzaher}
\email{zaher@illinois.edu}
\affiliation{%
  \institution{University of Illinois Urbana-Champaign}
  \city{Urbana}
  \state{IL}
  \country{USA}
}
\usepackage{algorithm2e} 
\usepackage{amsfonts}    
\usepackage{amsmath}     
\usepackage{bm}          
\usepackage{booktabs}    
\usepackage{caption}     
\usepackage{chapterbib}  
\usepackage[super]{nth}  
\usepackage{diagbox}     
\usepackage{enumitem}    
\usepackage{epstopdf}    
\usepackage[T1]{fontenc} 
\usepackage{fancyhdr}    
\usepackage{float}       
\usepackage{graphicx}    
\usepackage{hyperref}    
\usepackage[utf8]{inputenc} 
\usepackage{lipsum,multicol} 
\usepackage{makecell}    
\usepackage{microtype}   
\usepackage{multirow}    
\usepackage{nicefrac}    
\usepackage[normalem]{ulem}
\usepackage{placeins}    
\usepackage{setspace}    
\usepackage{slashbox}    
\usepackage{soul}        
\usepackage{subfigure}   
\usepackage{tabularx}    
\usepackage{textcomp}    
\usepackage{tikz}        
\usepackage{ulem}        
\usepackage{url}         
\usepackage{verbatim}    
\usepackage{wrapfig}     
\usepackage{xcolor}      

\usepackage{pifont}
\usepackage{hyperref}

\useunder{\uline}{\ul}{}
\def \eg {\emph{e.g.}, }

\def \et {\emph{et al.}}

\newcommand{\indep}{\perp \!\!\! \perp}

\newcommand{\model}{{InfoMAE}\xspace}
\newcommand{\paradigm}{{InfoMAE}\xspace}
\newcommand{\paradigmabbrv}{{InfoMAE}\xspace}

\newcommand{\xmark}{\ding{55}}%

\renewcommand{\shortauthors}{Tomoyoshi Kimura et al.}



\begin{document}
\begin{abstract}
Standard multimodal self-supervised learning (SSL) algorithms regard cross-modal synchronization as implicit supervisory labels during pretraining, thus posing high requirements on the scale and quality of multimodal samples.
These constraints significantly limit the performance of sensing intelligence in IoT applications, as the heterogeneity and the non-interpretability of time-series signals result in abundant unimodal data but scarce high-quality multimodal pairs. 
This paper proposes \model, a cross-modal alignment framework that tackles the challenge of multimodal pair efficiency under the SSL setting by facilitating efficient cross-modal alignment of pretrained unimodal representations. 
\model achieves \textit{efficient cross-modal alignment} with \textit{limited data pairs} through a novel information theory-inspired formulation that simultaneously addresses distribution-level and instance-level alignment.
Extensive experiments on two real-world IoT applications are performed to evaluate \model's pairing efficiency to bridge pretrained unimodal models into a cohesive joint multimodal model. 
\model enhances downstream multimodal tasks by over 60\% with significantly improved multimodal pairing efficiency. It also improves unimodal task accuracy by an average of 22\%
\footnote{The code is available at https://github.com/tomoyoshki/InfoMAE.}.
\end{abstract}
\maketitle

\section{Introduction}

Multimodal Self-Supervised Learning (SSL) algorithms, although achieving unprecedented performance in extensive sensing applications~\cite{deldari2022cocoa, delpreto2022actionsense, piechocki2023multimodal, kimura2024vibrofm}, present unique data challenges rarely encountered with unimodal SSL or vision-language domains due to the complexity in acquiring high-quality multimodal pairs for IoT signals.
The inherent properties of sensory data common in Web and Industrial sensing applications result in abundant unimodal signals but scarce multimodal pairs.
First, sensory modalities have heterogeneous properties, such as sampling rate, timestamp, or duration, that increase the likelihood of capturing asynchronous events. 
For example, in vibration sensing applications (machine monitoring, vehicle detection), multimodal sensors (geophone, microphone, thermometer, etc.) often operate at different sampling rates, leading to temporal misalignments that require manual calibration~\cite{liu2024focal}.
Second, raw signals often lack intuitive interpretability. 
Unlike images or text, where visual features can be easily matched to textual captions, capturing useful signatures between sensing modalities like motion or frequency waves is challenging.
Preprocessing and calibrating these signals requires modality-specific domain knowledge, which is labor-intensive and susceptible to operational errors.
Finally, sensors for IoT are subject to varying deployment conditions, leading to sparse and noisy data~\cite{li2024acies}.
For example, in human activity recognition (HAR) applications, wearable IMU sensors generate multimodal motion streams for real-time monitoring, fitness tracking, or healthcare purposes. Each modality can be independently affected by device constraints, platform heterogeneity, sensor failures, or variations in deployment environments, leading to missing or incomplete data streams.
This heterogeneity often yields poor-quality uncorrelated multimodal pairs or incomplete datasets with significant gaps and missing data.
As IoT networks scale in quantity and the number of modalities, acquiring large-scale, high-quality multimodal pairs becomes increasingly time-consuming, error-prone, and less reliable.

Despite these challenges, most existing multimodal SSL frameworks~\cite{alwassel2020self, korbar2018cooperative, morgado2020learning, radford2021learning} rely heavily on massive multimodal pairs to learn robust joint representations during the pretraining, but their capability could degrade significantly with insufficient synchronized pairs~\cite{ma2022multimodal, wang2023connecting} or uninformative false-positive pairs~\cite{chuang2022robust, morgado2021robust}
On the other hand, independently pretraining each modality on their unimodal data and directly concatenating misaligned modality features for finetuning fails to capture cross-modal interactions that are critical to downstream multimodal tasks~\cite{huang2021learning, verma2019learning}. 
\begin{figure}[t!]
\centering
\includegraphics[width=0.95\columnwidth]{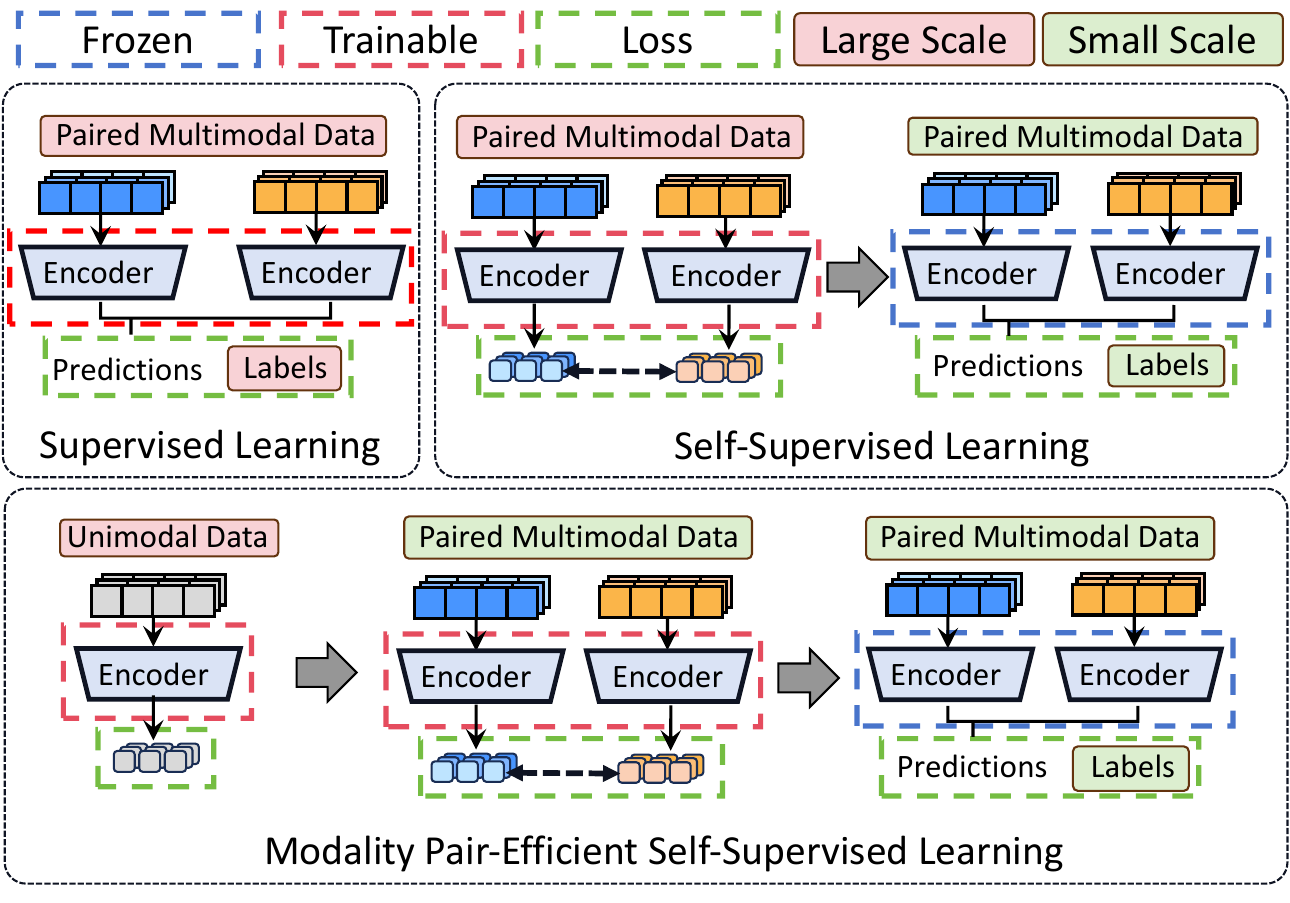}
\vspace{-0.4cm}
\caption{
Comparison of supervised learning, self-supervised learning, and pair-efficient self-supervised learning.
}
\label{fig:learning_paradigm}
\end{figure} 
Instead, we observe that with limited multimodal pairs, we can effectively convert independently trained unimodal encoders into a coherent model that sustains strong generalizability in multimodal tasks. We refer to this process as \textit{pair-efficient SSL}.
The relation of pair-efficient SSL for multimodal data compared to standard SSL draws an analogy to the evolvement of SSL compared to supervised learning, as visualized in Figure \ref{fig:learning_paradigm}. 
In supervised learning, manual labels serve as supervision to train encoders for mapping inputs to task-specific labels.
Its performance depends heavily on the quantity and quality of human annotations. 
Self-supervised learning (SSL) mitigates label scarcity by first designating proxy labels from the data properties to learn general semantics with massive unlabeled data, then calibrating the pretrained model to a downstream task with minimal human annotations.
Similarly, in multimodal SSL contexts, cross-modal alignment acts as a special form of ``supervision'', where point-to-point modality correspondence is utilized to identify semantically meaningful and consistent sensory information. 
Taking another step forward, pair-efficient SSL takes advantage of abundant unimodal data for ``independent pretraining'', followed by ``cross-modal finetuning'' with limited multimodal pairs to align unimodal models into a cohesive multimodal model.

In this paper, we propose InfoMAE, a cross-modal learning framework designed to enhance the alignment of unimodal representations using a limited number of multimodal pairs. 
The key idea behind InfoMAE is to enforce alignment across modalities at both the \textit{distribution} and \textit{instance} levels.
Existing contrastive learning frameworks adopt point-to-point alignment to map samples across different modalities to a proximate joint representation~\cite{tian2020contrastive, liu2024focal, ouyang2022cosmo, poklukar2022gmc}. 
These approaches focus on aligning individual samples, essentially viewing alignment as a local optimization problem that aims to minimize the geometric distances between corresponding samples in the representation space.
However, such instance-level approaches face significant challenges with limited multimodal pairs, as they may overfit to the specific pairs available and result in poor generalization with pairing biases.
These hinder capturing complex cross-modal relationships, especially when the multimodal pairs are sparse and unevenly distributed.
In contrast, InfoMAE takes a more holistic approach by emphasizing \textit{distribution-level} alignment, considering the overall information content of the limited multimodal pairs rather than only focusing on the individual samples.
We present a comprehensive analysis of distribution alignment and propose an \textit{information theory-based approach} to formally define the distribution alignment problem in the factorized information space. 
We formulate this as a differential learning objective to construct (i) shared joint representations as a compact common variable across modalities capable of performing any multimodal task and (ii) private representations holding implicit modality-specific information independent of shared representations. 
\model alleviates the strict requirement of exact multimodal sample pairs and can better accommodate potential misalignments in data collection or temporal synchronization, improving the representations learned even with a small-scale multimodal pair. 

We extensively evaluate \model across various combinations of pretrained unimodal domains. 
\model achieves exceptional performance gain compared to the standard multimodal SSL paradigm under limited multimodal pairs and outperforms existing works when aligning the unimodal representations.
Individual unimodal encoders, in return, can also benefit from the representational structures with improved downstream performance.
Additionally, as the number of multimodal pairs scale, \model also demonstrates versatility as a standard multimodal SSL framework, achieving SOTA performance across real-world IoT applications.

\section{Analysis of Cross-Modal Alignment}\label{sec:formulation}

\subsection{Notation}
Consider $M$ sets of unsynchronized sensory modality data $\mathcal{X} = {\{X_i\}}_{i \in M}$, where each set $X_{i}$ contains unlabeled samples of fixed-length windows partitioned from the time-series signals of the $i$-th sensory modality. Let $N_i = |X_i|$ denote the size of each set.

For the $j$-th sample of modality set $i$, we apply Short-Time Fourier Transform (STFT) to obtain its time-frequency representation, $\mathbf{x}_{ij} \in \mathbb{R}^{C_i \times I\times S_i}$, where $C_i$ is the number of input channels, $I$ is the number of time intervals within a sample window, and $S_i$ is the spectrum length in the frequency domain.
We have a set of modality encoders $\mathcal{E} = \{E_1, E_2, \dots, E_M\}$ to extract the modality embeddings of each sample and a set of modality decoders $\mathcal{D} = \{D_1, D_2, \dots, D_M\}$ to map the samples from the embedding space back to the time-frequency domain $\mathcal{\hat{X}} = {\{ \hat{X_i} \}}_{i \in M}$ as a part of the reconstruction process.
Additionally, there is a set of multimodal data $\mathcal{X}^s = \{X^s_i\}_{i \in M^s}$ consisting of a subset of modalities $M^s \subseteq \hat{M}$, where samples across the modalities are synchronized in time and have equal sizes $|X_1^s| =\cdots= |X^s_{M^s}|$.
Note that each synchronized data of modality $i$ can also be a subset of the unsynchronized unimodal set such that $X^s_i \subseteq X_i$, as any synchronized multimodal data is inherently unsynchronized when considered independently.
Finally, we have a set of labeled data for supervised learning and finetuning on a much smaller scale, where each sample has a corresponding label $y_j$ for each downstream task.

\subsection{Problem Definition}
Prior multimodal SSL practices rely on large-scale, fully synchronized multimodal sets $\mathcal{X}^s$ to learn joint multimodal representations for downstream tasks.
However, these approaches overlook two challenges:
(i) \textit{Insufficient multimodal data}: When $|\mathcal{X}^s|$ is small, existing methods struggle to learn effective joint representations, and (ii) \textit{Unutilized unimodal data}: The abundance of unimodal data is often ignored.
In IoT applications, synchronized multimodal sets are limited due to signal heterogeneities, temporal misalignment, or domain variances, leading to incomplete modalities.
This results in limited synchronized multimodal data compared to unimodal data ($|X^s_i| \le |X_i|$). 
To better leverage unimodal data, our problem falls under the SSL setting with unimodal pretrained models and limited multimodal pairs, consisting of two stages:

\textit{Stage 1: Independent Unimodal Pretraining.} For each independent modality data $X_i$, we train a corresponding unimodal encoder $E_i$. The goal is to learn a \textit{holistic unimodal representation} that maximizes downstream unimodal performance after finetuning.
Since modality sets $X_i$ are independent, this pretraining is not limited by the number of synchronized pairs and can, therefore, fully leverage the abundant unimodal data.

\textit{Stage 2: Efficient Cross-Modal Alignment.} Given a set of synchronized modalities data $\mathcal{X}^s$ of $M^s \subseteq M$ modalities, we aim to align the pretrained encoders efficiently. This alignment projects unimodal representations into joint representations that maximize the downstream multimodal performance after finetuning. The scale of the multimodal alignment should be significantly smaller than the unimodal pretraining $|X^s_i| \llless |X_i|$. 
In contrast to prior multimodal SSL works focusing on learning robust joint representations on large-scale multimodal data, this work aims to improve the \textit{data efficiency} of learning robust joint representations given only limited multimodal pairs.

\subsection{Factorization \& Distributional Alignment}
\begin{figure}[!t]
\centering
\includegraphics[width=0.75\columnwidth]{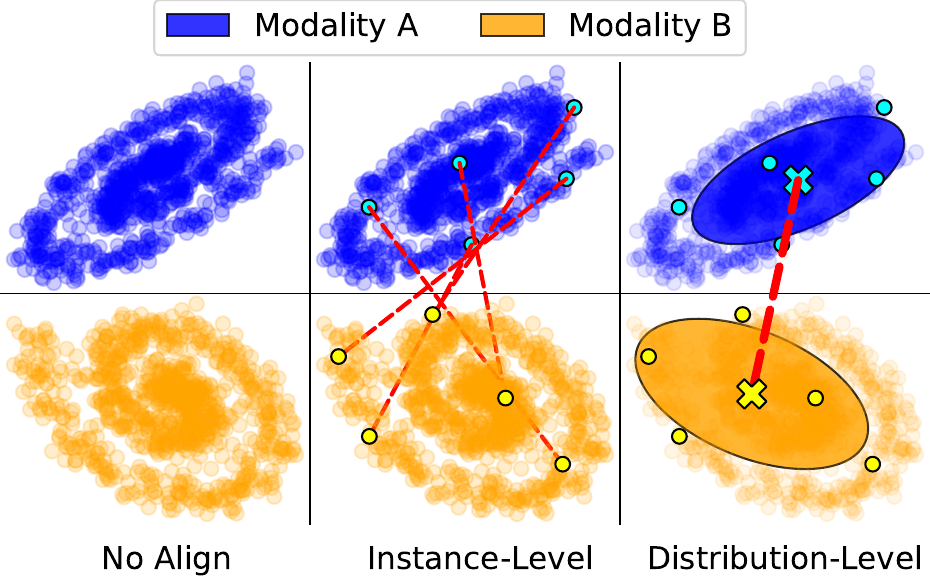}
\caption{An illustration of instance-level vs. distribution-level Cross-Modal Alignment}
\label{fig:alignment_illustration}
\end{figure}

This section analyzes multimodal representation factorization in the information space and demonstrates how it enables distribution-level alignment of unimodal representations.

\subsubsection{Connection between Factorization and Cross-modal Alignment}
In aligning multimodal representations, prior approaches often rely on contrastive learning to minimize the \textit{modality gap}~\cite{liang2022mind} by pulling representations of different modalities from the same sample closer together while pushing representations from different samples further apart.
However, due to the inherent heterogeneity, each modality contains unique, modality-specific information, and enforcing perfect alignment across modalities could potentially hurt the performance in multimodal downstream tasks~\cite{jiang2023understanding}. 
To address these challenges, recent works \cite{liu2024focal, jiang2023understanding, liang2024factorized} have proposed factorizing modality representations into shared and private subspaces.
It preserves both common and modality-specific information and allows for the alignment of shared representations while maintaining independent private representations for downstream tasks.
However, these works operate on \textit{instance-level alignment} and do not explore scenarios with limited multimodal data.
The scarcity of paired samples introduces the risk of biased sampling, potentially misleading the alignment process.
With this in mind, we analyze a different approach that factorizes the representation in the information space and enforces \textit{distribution-level} alignment to capture a more comprehensive correlation between modalities by \textit{emphasizing their information content rather than just their geometric proximity}.
The intuition behind this is that instead of individual sample pairs, we aim to align modalities by the global structure (as shown in Figure~\ref{fig:alignment_illustration}). 
When the multimodal pairs are scarce, the distributional alignment aims to be \textit{resilient to sampling biases} and capture meaningful cross-modal relationships.

\subsubsection{Distributional Alignment through Information-theory based Factorization}\label{subsec:info}
We now formally define the factorization problem in the information space.
Without loss of generality, we state the definitions for two modalities, $\mathcal{X} = \{X_1, X_2\}$, but they can be generalized to more modalities. 

First, we are interested in constructing a compact random variable $U$ (shared representation) that can perform any task that can be achieved using $X_1$ separately and $X_2$ separately. Formally, we define a sufficient common variable as follows.

\begin{definition} (Sufficient Common Variable) $U$ is defined as the sufficient common variable between $X_1,X_2$ if and only if $U = g_1(X_1)=g_2(X_2)$ for some $g_1,g_2$, and

{\small{
\begin{equation}\label{eq:def-1}
    (\forall f_1,f_2) \Big([f_1(X_1)=f_2(X_2)]\implies  [(\exists f) f(U) = f_1(X_1) = f_2(X_2)]\Big),
\end{equation}
}}

\noindent namely, any common (shared) function between $X_1,X_2$ can be computed using $U$.
Building on the sufficient common variable, we define the shared representation to be the most compact form of $U$ with the minimized entropy to ensure that $U$ captures only the essential shared features across modalities. 
\end{definition}

\begin{definition}\label{definition:shared}
(Shared Representation) We refer to a sufficient common variable $U$ with minimal entropy $H(U)$ as the shared representation. 
\end{definition}
However, it is not clear how to find a sufficient common variable or a shared representation. We show that an approximation of the shared representation can be obtained by solving the following optimization problem, and later in Section~\ref{sec:paradigm}, we propose the differentiable loss objectives with proof provided in Appendix~\ref{appendix:formulation}. 
{{
\begin{equation}\label{eq:def-2}
        \text{min } H(U) \text{ s.t. } X_1 \indep X_2 \mid U, (\exists s_1,s_2)\ U = s_1(X_1)=s_2(X_2)
\end{equation}
}
}
The conditional independence in Equation \ref{eq:def-2} enforces a form of distributional alignment, ensuring that given the shared representation $U$ is the most compact aligned representation such that $X_1, X_2$ provide no additional information about each other.
Moreover, we define the private representations $V_1, V_2$ between $X_1, X_2$ as follows.
\begin{definition}\label{definition:private} (Private Representation)
$V_1, V_2$ is the private representation of $X_1, X_2$ if they have minimal entropy among the random variables satisfying: $V_1 = p_1(X_1), V_2 = p_2(X_2)$ for some $p_1,p_2$ and there exist functions $g_1,g_2$ such that $X_1=g_1(V_1,U), X_2=g_2(V_2,U)$, where $U$ is the shared representation.
\end{definition}

Similarly, we look for approximate representations. In particular, we replace equalities with a distance constraint $d$, and independence is replaced by small mutual information. 
In Section~\ref{sec:paradigm}, we discuss the detailed implementation of a differentiable loss function to find the approximate representations.


\begin{figure*}[t!]
\centering
\includegraphics[width=0.9\textwidth]{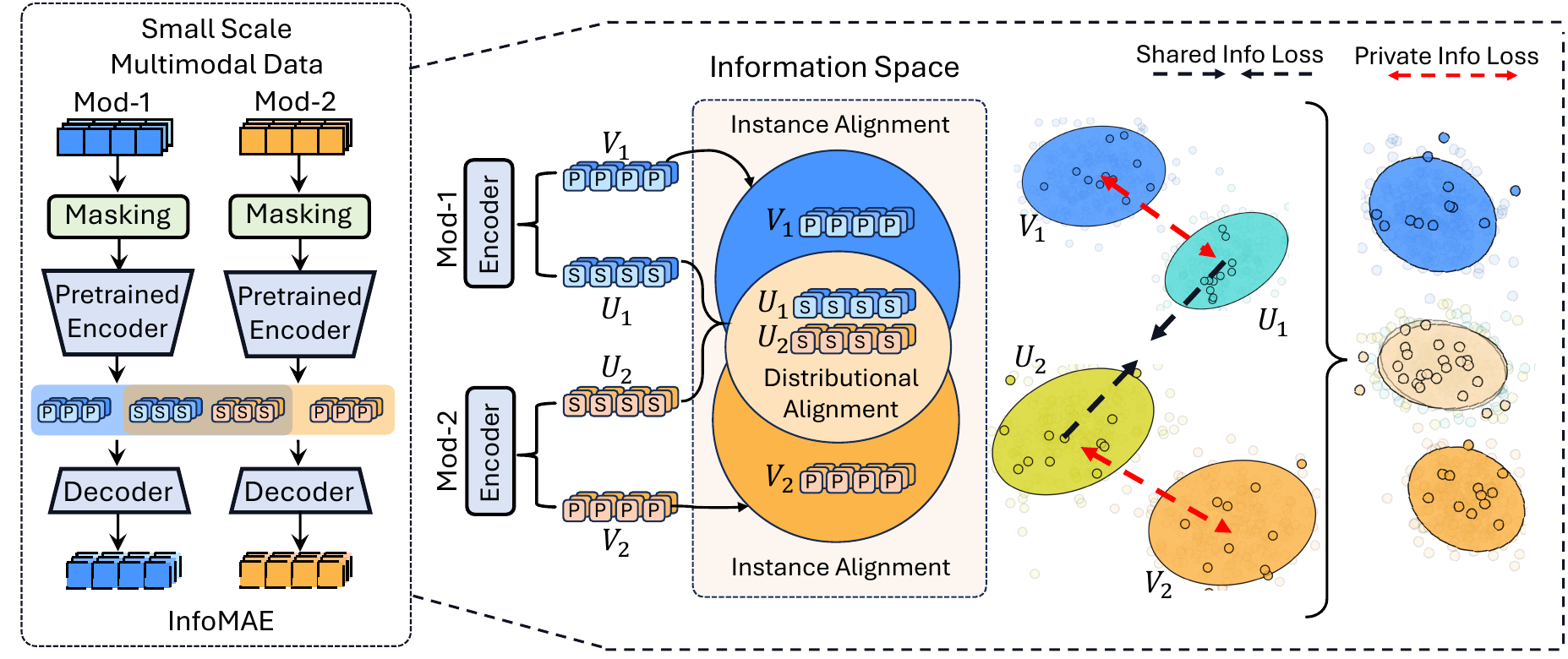}
\caption{
Overview of \model's alignment in the information space. \model adopts an information theory-inspired objective to align the factorized representations. Best viewed in color. 
}
\label{fig:overview_framework}
\vspace{-0.1cm}
\end{figure*}

\section{\model}\label{sec:paradigm}
This section introduces InfoMAE, a novel cross-modal alignment framework that efficiently aligns unimodal representations at the distribution and instance levels.
We provide a detailed overview of \model's cross-modal alignment module in Figure~\ref{fig:overview_framework}. 
\subsection{Unimodal Pretraining}
Unlike standard multimodal SSL that pretrains on synchronized multimodal pairs, we first initiate \textit{unimodal pretraining} on large-scale unsynchronized unimodal data. 
In the first stage, we pretrain each encoder $E_i$ independently on unimodal data $X_i$ with masked reconstruction, defined as the following for each modality $i \in M$:
\begin{equation}
\mathcal{L}_i^\text{unimodal} = ||\hat X_i - X_i||^2 \mid \hat X_i = D_i(E_i(X_i)). 
\end{equation}

The pretrained unimodal encoders $E_i$ extract a generalized representation for each modality $M_i$. However, they do not guarantee information compatibility between modalities when used together in the downstream tasks.
In the following sections, we present \model's different components (as illustrated in Figure~\ref{fig:overview_module}) to calibrate the encoders to \textit{explicitly align} the modalities in both the distribution-level and instance-level with only a limited amount of multimodal pair $\mathcal{X}^s$.

\begin{figure}[t!]
\centering
\includegraphics[width=0.9\columnwidth]{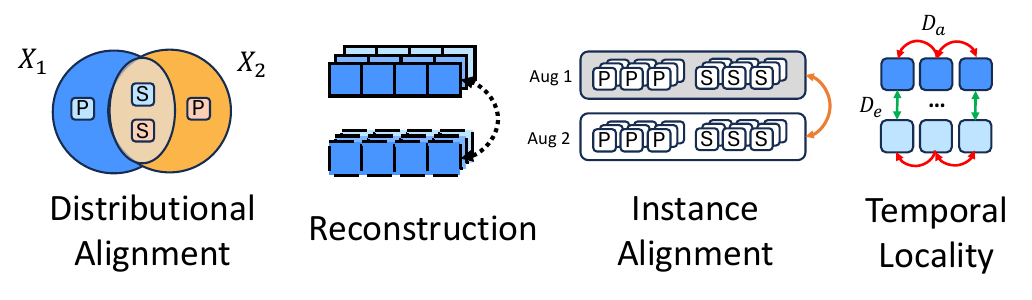}
\vspace{-0.5cm}
\caption{Key learning objectives of InfoMAE.}
\label{fig:overview_module}
\end{figure} 

\subsection{Distribution-level Alignment}\label{method:info}
We begin with the differentiable objective function that we optimize to obtain the (approximate) shared ($U$) and private representations ($V$) defined in Section \ref{subsec:info}. 
To extract $U$ that is a function of both $X_1,X_2$, we equivalently extract $U_1=F_1^\text{shared}(E_1(X_1)), U_2=F_2^\text{shared}(E_2(X_2))$, where $F_1, F_2$ are 2-layer MLP projectors that maps the general representation into factorized representations, and enforce a constraint that $U_1=U_2$.
Similarly, we extract $V_1=F_1^\text{private}(E_1(X_1)), V_2=F_2^\text{private}(E_2(X_2))$.
$\mathcal{U} = \{U_1, U_2\}$ and $\mathcal{V} = \{V_1, V_2\}$ denote the shared and private representations, respectively.

%

\subsubsection{Shared Representation}\label{method:common}
As described in Section~\ref{sec:formulation}, we aim to find the shared representation $U$ that solves the optimization problem in Definition~\eqref{definition:shared}.
However, due to the difficulty of the optimization problem \footnote{The optimization problem in Definition ~\eqref{definition:shared} is non-convex with a possibly infinite number of variables.} and the possibility that a shared representation does not exist, we instead approximate the shared representation by minimizing the following objective

{\small{
\begin{equation}\label{eq:loss-1}
\begin{aligned}
    \mathcal{L}_\text{info}^\text{shared} 
    = &\alpha d(U_1, U_2) + \beta (H(U_1) + H(U_2)) \\ 
    & + I(X_1; X_2 \mid U_1) + I(X_1; X_2 \mid U_2),
\end{aligned}
\end{equation}
}}

where $\alpha$ and $\beta$ are the hyperparameters controlling the weight of each term, and $d(\cdot)$ is a distance measure.
The first two terms in the loss function aim to find $U_1=U_2$ with minimal entropy, while the last two terms aim to impose conditional independence of $X_1,X_2$ given $U_1$ or $U_2$.
We would like to note that the entropy and conditional mutual information listed in Eq.~\eqref{eq:loss-1} are not easy to compute or differentiate. To alleviate this, we reduce these terms into probabilistic density functions below:

\begin{small}
\begin{equation}\label{equation:shared}
    \begin{aligned}
        \mathcal{L}_\text{info}^\text{shared}  &= \alpha d(U_1, U_2) + \sum_{i=1}^{2} \mathbb{E}_{X_1, X_2, U_i}\left[ \log \frac{p_{X_1, X_2, U_i}}{p_{X_1}{p_{X_2}}{p_{U_i}}} \right. \\
        &\left.+ (1 - \beta) \log{p_{X_i, U_i} \over p_{X_i}p_{U_i}} + \log{p_{X_{3-i}, U_i} \over p_{X_{3-i}}p_{U_i}}\right].
    \end{aligned}
\end{equation}
\end{small}

Due to the space limit, we leave the detailed proof and discussion in Appendix \ref{appendix:formulation}. 
To further enhance the differentiability of Eq.~\eqref{equation:shared} by avoiding directly computing the probabilistic density (\eg $\log \frac{p_{X_1, X_2, U_i}}{p_{X_1}{p_{X_2}}{p_{U_i}}}$), we follow \cite{nguyen2010estimating, sugiyama2012density, kim2018disentangling} and utilize the \textit{density-ratio trick} to train a discriminator $\mathcal{R}$, which given $X_1,X_2,U$, outputs the probability that $X_1,X_2,U$ are generated from $p_{X_1, X_2, U_i}$, instead of $p_{X_1}{p_{X_2}}{p_{U_i}}$. The density ratio can then be estimated as

{\small
\begin{equation}\label{equation:drt}
    \log \frac{p_{X_1, X_2, U_1}}{p_{X_1}{p_{X_2}}{p_{U_1}}} = \log{{\mathcal{R}(X_1; X_2; U_1})\over 1 - \mathcal{R}(X_1; X_2; U_1)}.
\end{equation}
}

We train the discriminators jointly with the encoders and describe the training for both in Appendix~\ref{appendix:training}.

%
\subsubsection{Private Representation}\label{method:private}
As the decoders take both the shared and private representations as input, the self-reconstruction objective would enforce the private representations $V$ to capture the implicit modality-specific information.
Following Definition \ref{definition:private}, we minimize the entropy of the private representations $(V_1, V_2)$.
In addition, for each modality, we expect the private and shared representations to be independent.
To better guide the learning process, we explicitly minimize their mutual information.
The objectives of the private representations can be summarized as the following:
\begin{equation}\label{eq:private}
\begin{aligned}
    \mathcal{L}_\text{info}^\text{private}  &= \gamma H(V_{1}) + \gamma H(V_{2}) + \epsilon I(V_1; U_1) + \epsilon I(V_2; U_2), \\
\end{aligned}
\end{equation}
\noindent where $\gamma$ and $\epsilon$ are used as the hyperparameters for private entropy and shared private independence. 
Similar to Eq.\eqref{equation:shared}, we apply \textit{density-ratio trick} (Eq.\eqref{equation:drt}) to estimate each term in Eq.~\eqref{eq:private}.

While the formulation effectively aligns modality representations within the information space, it depends on further learning objectives to ensure they are meaningful for downstream tasks.
Next, we will describe the additional components of \model that are designed to capture meaningful representations.

\subsection{Self Reconstruction}\label{method:reconstruction}

\model applies the masked reconstruction objective to enforce that the learned representation captures the critical semantical information through reconstruction loss.
Following MAE\cite{he2022mae}, we mask out 75\% of the patched input.
To ensure both the shared and private representation are meaningful, the decoder takes in the concatenated shared and private representations $\mathbf{h}_{ij} = \mathbf{u}_{ij}||\mathbf{v}_{ij}$ to reconstruct the input $\mathbf{\hat x}_{ij}$. 
We compute the MSE on the masked portion of the reconstructed $\mathbf{\hat x}_{ij}$ and the original input $\mathbf{x}_{ij}$ with $\delta$ as the hyperparameter and $D_i(\cdot)$ as the decoder for modality $i$. 
{\small{
\begin{equation}\label{reconstruction}
    \mathcal{L}_\text{reconstruction} = \delta\sum_{i \in M}\sum_{j \in B}
    ||\mathbf{x}_{ij} - \mathbf{\hat{x}}_{ij}||^2
    \mid \mathbf{\hat{x}}_{ij} = D_i(\mathbf{h}_{ij}).
\end{equation}
}}

\subsection{Instance-level Alignment}\label{method:augmentation}
Augmentations are primarily used to generate different views for private-space contrastive learning in most existing works \cite{liu2024focal, liang2024factorized, jiang2023understanding}.
However, we argue that the transformation invariance property should be reflected in both private and shared representations to understand the instance variances. 
Thus, InfoMAE adds a contrastive loss on the concatenated representation of the shared and private spaces $\mathbf{h}_{ij}$ by treating two randomly different augmented views as the positive pairs with $\lambda$ and $\tau$ as the hyperparameters.
{\small{
\begin{equation}
    \mathcal{L}_\text{aug} = 
     \lambda\sum_{i \in M}\sum_{j \in B}\log{
    {\exp \left({\mathbf{h}_{ij} \cdot \mathbf{h'}_{ij}} / \tau \right)}
    \over
    {\sum_{k \neq j \in B} \exp \left({{\mathbf{h}_{ij} \cdot \mathbf{h}_{ik}} \over \tau} \right) + \sum_{k \in B} \exp \left({\mathbf{h}_{ij} \cdot \mathbf{h'}_{ik}} \over \tau \right)}
    }.
\end{equation}
}}
\vspace{-0.5cm}

\begin{table*}[!t]
\centering
\caption{
Linear probing performance of Moving Object Detection on domain M. We align pretrained unimodal encoders from different domains.
$A_{Sei} || B_{Aco}$ means seismic encoder from domain A and acoustic encoder from domain B are aligned.
}
\label{tab:synchronization}
\resizebox{0.85\textwidth}{!}{%
\begin{tabular}{@{}c|cc|cccccccccc@{}}
\toprule
\multirow{2}{*}{Framework} &
  \multicolumn{2}{c|}{Aligned Domains} &
  \multicolumn{2}{c|}{$T_{Sei}\ || \ M_{Aco}$} &
  \multicolumn{2}{c|}{$G_{Sei}\ || \ T_{Aco}$} &
  \multicolumn{2}{c|}{$T_{Sei}\ || \ T_{Aco}$} &
  \multicolumn{2}{c|}{$G_{Sei}\ || \ M_{Aco}$} &
  \multicolumn{2}{c}{$T_{Sei}\ || \ G_{Aco}$} \\ \cmidrule(l){2-13} 
 &
  \multicolumn{1}{c|}{\begin{tabular}[c]{@{}c@{}}Joint\\ Pretrain\end{tabular}} &
  \begin{tabular}[c]{@{}c@{}}Modal\\ Alignment\end{tabular} &
  Acc &
  \multicolumn{1}{c|}{F1} &
  Acc &
  \multicolumn{1}{c|}{F1} &
  Acc &
  \multicolumn{1}{c|}{F1} &
  Acc &
  \multicolumn{1}{c|}{F1} &
  Acc &
  F1 \\ \midrule
Unimodal Concat &
  \multicolumn{1}{c|}{\xmark} &
  \xmark &
  0.6731 &
  \multicolumn{1}{c|}{0.6699} &
  0.5392 &
  \multicolumn{1}{c|}{0.5281} &
  0.4454 &
  \multicolumn{1}{c|}{0.4366} &
  0.7247 &
  \multicolumn{1}{c|}{0.7217} &
  0.6584 &
  0.6543 \\ \midrule
CMC~\cite{tian2020contrastive} &
  \multicolumn{1}{c|}{\xmark} &
  \checkmark &
  0.6792 &
  \multicolumn{1}{c|}{0.6702} &
  0.4313 &
  \multicolumn{1}{c|}{0.4356} &
  0.4173 &
  \multicolumn{1}{c|}{0.4032} &
  0.6919 &
  \multicolumn{1}{c|}{0.6877} &
  0.6497 &
  0.6335 \\
FOCAL~\cite{liu2024focal} &
  \multicolumn{1}{c|}{\xmark} &
  \checkmark &
  0.7462 &
  \multicolumn{1}{c|}{0.7432} &
  0.6249 &
  \multicolumn{1}{c|}{0.6249} &
  0.5613 &
  \multicolumn{1}{c|}{0.5579} &
  0.7549 &
  \multicolumn{1}{c|}{0.7527} &
  0.7194 &
  0.7160 \\
GMC~\cite{poklukar2022gmc} &
  \multicolumn{1}{c|}{\xmark} &
  \checkmark &
  0.7354 &
  \multicolumn{1}{c|}{0.7317} &
  0.6591 &
  \multicolumn{1}{c|}{0.6523} &
  0.4756 &
  \multicolumn{1}{c|}{0.4720} &
  0.8044 &
  \multicolumn{1}{c|}{0.8053} &
  0.7247 &
  0.7211 \\
SimCLR~\cite{chen2020simclr} &
  \multicolumn{1}{c|}{\xmark} &
  \checkmark &
  0.3061 &
  \multicolumn{1}{c|}{0.2742} &
  0.2873 &
  \multicolumn{1}{c|}{0.2609} &
  0.2974 &
  \multicolumn{1}{c|}{0.2758} &
  0.2981 &
  \multicolumn{1}{c|}{0.2698} &
  0.2800 &
  0.2308 \\
TNC~\cite{tonekaboni2021unsupervised} &
  \multicolumn{1}{c|}{\xmark} &
  \checkmark &
  0.1969 &
  \multicolumn{1}{c|}{0.0815} &
  0.1788 &
  \multicolumn{1}{c|}{0.1312} &
  0.1855 &
  \multicolumn{1}{c|}{0.1021} &
  0.1929 &
  \multicolumn{1}{c|}{0.0896} &
  0.1949 &
  0.1041 \\
TSTCC~\cite{eldele2021time} &
  \multicolumn{1}{c|}{\xmark} &
  \checkmark &
  0.3001 &
  \multicolumn{1}{c|}{0.2706} &
  0.2639 &
  \multicolumn{1}{c|}{0.2393} &
  0.2867 &
  \multicolumn{1}{c|}{0.2432} &
  0.3048 &
  \multicolumn{1}{c|}{0.2842} &
  0.2860 &
  0.2337 \\ \midrule
\textbf{InfoMAE} &
  \multicolumn{1}{c|}{\xmark} &
  \checkmark &
  \textbf{0.7950} &
  \multicolumn{1}{c|}{\textbf{0.7929}} &
  \textbf{0.6986} &
  \multicolumn{1}{c|}{\textbf{0.7007}} &
  \textbf{0.5928} &
  \multicolumn{1}{c|}{\textbf{0.5908}} &
  \textbf{0.8326} &
  \multicolumn{1}{c|}{\textbf{0.8324}} &
  \textbf{0.7636} &
  \textbf{0.7537} \\ \midrule
Joint Pretrain &
  \multicolumn{1}{c|}{\checkmark} &
  \xmark &
  \multicolumn{5}{c|}{Acc: 0.3329} &
  \multicolumn{5}{c}{F1: 0.3039} \\ \bottomrule
\end{tabular}%
}
\end{table*}

\subsection{Temporal Locality}\label{method:temporal} 
We apply a simple ranking constraint to learn \textit{temporal locality} of time-series signals. 
During pretraining, a sequence sampler randomly selects a batch of sequences consisting of a fixed number of consecutive samples, while the samples across sequences are distant in time.
We define $C_{xy'} = \sum_{i=1}^{L}\sum_{j=1}^{L} d(x_i, y_j),$ as the average Euclidean distance (d) of all sample embedding pairs between the sequence $s$ and $s'$ of length $L$.
Then, the temporal constraint with a hyperparameter $\eta$ can be defined as:: 
{\small{
\begin{equation}
    \begin{aligned}
    \mathcal{L}_{\text{temp}} &= 
    \eta\sum_{s \in B} \sum_{s' \neq s \in B}  \max{(C_{ss} - C_{ss'} + 1, 0)}
    \end{aligned}
\end{equation}
}}
where $C_{ss}$ and $C_{ss'}$ measure the average intra-sequence ($D_a$) and inter-sequence ($D_e$) distances . The added 1 is the margin indicating the minimum gap between the two distances. 
$\eta$ is used as the hyperparameter to control the weight of the temporal constraint.


Finally, the overall training objective of InfoMAE for the cross-modal alignment stage can be summarized as follows:
{\small{
\begin{equation}\label{overall}
    \mathcal{L} = 
    \mathcal{L}_\text{info}^\text{shared}  + 
    \mathcal{L}_\text{info}^\text{private}  + 
    \mathcal{L}_{\text{reconstruction}} + 
    \mathcal{L}_{\text{aug}} +
    \mathcal{L}_{\text{temp}}.
\end{equation}
}}
\model adopts both distribution-level and instance-level alignment of each modality's factorized shared and private representations. 
Since the cross-modal alignment of \model is also a generalized multimodal framework, we would also like to note that this objective can be used as the joint multimodal pretraining objective. 

\section{Evaluation}\label{sec:eval}


\subsection{Experimental Setup}\label{subsec:exp_setup}

\subsubsection{Backbone Encoder}
We adopt the SWIN Transformer (SW-T) \cite{liu2021swin} as the backbone encoder for our framework. SW-T computes local attention within shifted windows on input spectrogram patches to extract comprehensive time-frequency representations.

\subsubsection{Datasets} 
Our experiments focus on Moving Object Detection (MOD) and Human Activity Recognition (HAR). 
The MOD application contains vibration-based datasets using seismic and acoustic sensors.
The HAR application consists of publicly released IMU sensor datasets collected from human subjects performing various daily activities.
To evaluate cross-modal alignment, we simulate a practical scenario where the pretrained domains differ significantly to reflect the diverse signals across different IoT domains.
Under this setting, we have unsynchronized unimodal data from different domains: 
MOD consists of data from three separately collected domains ($M$, $G$, $T$), each with different targets, terrains, and environmental conditions.
HAR consists of two datasets (RW-HAR~\cite{sztyler2016body} and PAMAP2~\cite{reiss2012introducing}). 
We pretrain unimodal encoders with only the unimodal data from each domain and then use small-scale synchronized multimodal pairs for cross-modal alignment.
For joint pretraining, we pretrain on the massive available synchronized multimodal pairs.
We summarize and describe these applications and domains in more detail in Appendix~\ref{appendix:datasets}.

\begin{table*}[]
\centering
\caption{Alignment performance (MM) with different multimodal pair ratios from MOD.}
\label{tab:multimodal_ratio}
\resizebox{0.8\textwidth}{!}{%
\begin{tabular}{@{}c|cc|cc|cc|cc|cc|cc@{}}
\toprule
\multirow{2}{*}{Multimodal Data} &
  \multicolumn{2}{c|}{Supervised} &
  \multicolumn{2}{c|}{Joint Pretrain} &
  \multicolumn{2}{c|}{CMC} &
  \multicolumn{2}{c|}{GMC} &
  \multicolumn{2}{c|}{FOCAL} &
  \multicolumn{2}{c}{\textbf{InfoMAE}} \\ \cmidrule(l){2-13} 
     & Acc & F1 & Acc    & F1     & Acc    & F1     & Acc    & F1     & Acc    & F1     & Acc             & F1              \\ \midrule
5\% &
  \multirow{4}{*}{0.5740} &
  \multirow{4}{*}{0.5663} &
  0.3329 &
  0.3039 &
  0.7087 &
  0.6989 &
  0.8614 &
  0.8616 &
  0.8694 &
  0.8668 &
  \textbf{0.8828} &
  \textbf{0.8808} \\
15\% &     &    & 0.6142 & 0.6104 & 0.8111 & 0.8062 & 0.8781 & 0.8753 & 0.8727 & 0.8703 & \textbf{0.9049} & \textbf{0.9028} \\
25\% &     &    & 0.7071 & 0.7938 & 0.8433 & 0.8372 & 0.8774 & 0.8759 & 0.8848 & 0.8831 & \textbf{0.9290} & \textbf{0.9270} \\
50\% &     &    & 0.8942 & 0.8920 & 0.8754 & 0.8724 & 0.8948 & 0.8938 & 0.9009 & 0.8994 & \textbf{0.9377} & \textbf{0.9367} \\ \bottomrule
\end{tabular}%
}
\end{table*}

\subsubsection{Baselines} We compare \model with different SOTA SSL baselines including unimodal CL (SimCLR\cite{chen2020simple}, MoCo\cite{chen2021mocov3}), multimodal CL (CMC\cite{tian2020contrastive}, GMC\cite{poklukar2022gmc}, FOCAL \cite{liu2024focal}), temporal CL (TNC\cite{tonekaboni2021unsupervised}, TSTCC\cite{eldele2021time}), and MAE based frameworks (MAE\cite{he2022mae}, CAV-MAE\cite{gong2022contrastive}).


\subsection{Cross-Modal Alignment Evaluation}\label{sec:eval_paradigm}

\subsubsection{Moving Object Detection}
We evaluate \model against prior CL works \cite{chen2020simple, eldele2021time, liu2024focal, poklukar2022gmc, tian2020contrastive, tonekaboni2021unsupervised} on cross-modal alignment with various combinations of unimodal encoders (seismic and acoustic) pretrained with different domains.
We align the encoders with a small scale of multimodal pairs (5\% of the unimodal data scale) and an even smaller subset of labeled multimodal pairs from domain M for finetuning.
MOD application involves two modalities (seismic and acoustic). Therefore we represent the domains of the unimodal representations with two letters (\eg $T_{\text{Sei}}||G_\text{Aco}$ represents aligning the seismic encoder pretrained on domain T and acoustic encoder pretrained on domain G).

In addition to the prior CL baselines, we also show the performance for direct concatenation of the pretrained unimodal representations without any alignment and for Joint Multimodal Pretraining on the same amount of synchronized multimodal pairs.
We present the finetune accuracy and F1-score in Table \ref{tab:synchronization}, \model consistently outperforms the unimodal concatenation by a significant margin since direct concatenation fails to exploit cross-modal correspondence.
CMC and other unimodal SSL frameworks even have negative impacts compared to direct concatenation, indicating that unimodal objectives or simply aligning the multimodal representations without considering the modality discrepancy could hurt the downstream performance. 
\model also achieves better results than FOCAL and GMC, underscoring the benefits of enforcing distribution-level alignment over instance-level alignment in downstream tasks with limited multimodal data.
When the same amount of multimodal data is used for Joint Multimodal Pretraining, the significant gap between the aligned unimodal models and the joint pretrained multimodal model suggests the feasibility of transferring pretrained unimodal representations to multimodal representations with only limited (5\%) synchronized multimodal data. 
Note that some domain combinations (\eg, $G_\text{sei}||T_\text{aco}$, $T_\text{sei}||T_\text{aco}$, $T_\text{sei}||G_\text{aco}$) do not even overlap with the alignment and finetuning domain $M$. 

\begin{table}[]
\centering
\caption{
Linear probing performance of HAR on PAMAP2 by aligning pretrained unimodal encoders.
}
\label{tab:har_alignment}
\resizebox{\columnwidth}{!}{%
\begin{tabular}{@{}c|cc|cc|cc@{}}
\toprule
\begin{tabular}[c]{@{}c@{}}Unimodal\\ Pretrain Domain\end{tabular} &
  \multicolumn{2}{c|}{Combined} &
  \multicolumn{2}{c|}{PAMAP2} &
  \multicolumn{2}{c}{RW-HAR} \\ \midrule
\begin{tabular}[c]{@{}c@{}}Multimodal\\ Alignment Domain\end{tabular} &
  \multicolumn{2}{c|}{PAMAP2} &
  \multicolumn{2}{c|}{PAMAP2} &
  \multicolumn{2}{c}{PAMAP2} \\ \midrule
Metric & Acc    & F1     & Acc    & F1     & Acc    & F1     \\ \midrule
Concat & 0.7843 & 0.7000 & 0.7763 & 0.6210 & 0.5675 & 0.4187 \\
CMC    & 0.7334 & 0.6508 & 0.7285 & 0.6788 & 0.7010 & 0.5956 \\
FOCAL  & 0.7922 & 0.7129 & 0.7354 & 0.6327 & 0.7643 & 0.6243 \\
GMC    & 0.7314 & 0.5915 & 0.7344 & 0.5869 & 0.7414 & 0.5816 \\
SimCLR & 0.7299 & 0.6190 & 0.7075 & 0.5426 & 0.7225 & 0.5581 \\
TNC    & 0.5431 & 0.4080 & 0.5889 & 0.4824 & 0.6378 & 0.5167 \\
TSTCC  & 0.7299 & 0.6003 & 0.7065 & 0.5773 & 0.7354 & 0.5864 \\ \midrule
\textbf{InfoMAE} &
  \textbf{0.8261} &
  \textbf{0.7303} &
  \textbf{0.8117} &
  \textbf{0.7175} &
  \textbf{0.7912} &
  \textbf{0.6901} \\ \bottomrule
\end{tabular}%
}
\end{table}

\subsubsection{Human Activity Recognition}
Besides MOD application, we also evaluate \model on HAR applications.
In contrast to MOD evaluation, which aligns unimodal encoders pretrained on different domains, we analyze how additional unsynchronized data from the same domains could assist the downstream performance given the limited number of multimodal pairs.
Here, we independently pretrain all unimodal encoders on unsynchronized IMU data from either PAMAP2, RW-HAR, or Combined, which is the concatenation of the former two.
Then, we use a small portion of the synchronized multimodal data pairs from PAMAP2 for cross-modal alignment and downstream finetuning. 
We present the results in Table~\ref{tab:har_alignment}. 
\model consistently achieves the best performance, with an average of 4.09\% and 5.16\% improvements in accuracy and the F1-score compared to the best-performing baseline, FOCAL. 
The improvement is most significant in aligning unimodal encoders pretrained on RW-HAR, which completely differs from the alignment set (PAMAP2).
This further demonstrates \model's robustness as an alignment framework with a limited amount of multimodal pairs, reflecting its superior ability to utilize the unimodal data better even when they are from different domains.

\begin{figure}[!t]
\centering
\includegraphics[width=0.47\textwidth]{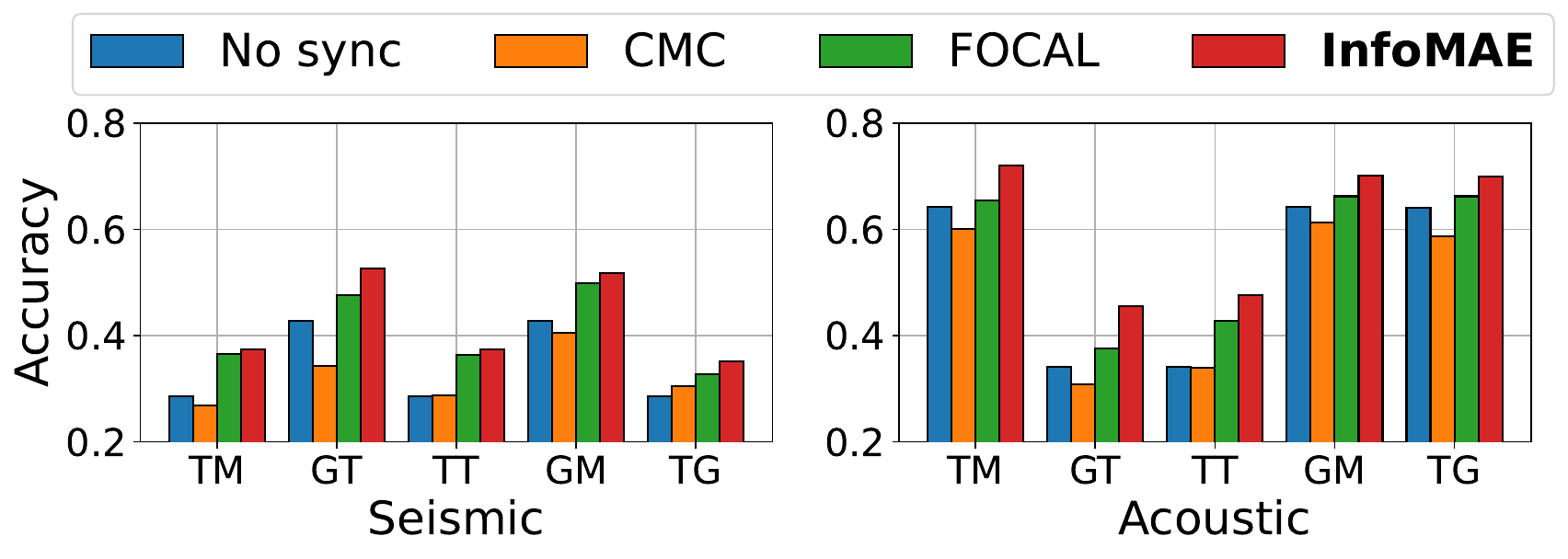}
\vspace{-0.4cm}
\caption{Unimodal linear probing accuracy of MOD with and without cross-modal alignment.}
\label{fig:unimodal}
\end{figure}

\subsection{Unimodal Evaluation} We analyze how incorporating the multimodal correspondences into each unimodal encoder after alignment could benefit the downstream tasks. 
Figure \ref{fig:unimodal} shows the accuracy for seismic and acoustic modalities before and after cross-modal alignment in the MOD application.
With limited multimodal pairs, the pretrained unimodal encoders could gain the most significant performance improvements with \model.
This emphasizes the \model's superior efficiency in enforcing cross-modal correspondence to each modality to improve their downstream performance, with only a few multimodal pairs required. 
With \model, the aligned unimodal model can generate the most holistic representations through distributional alignment compared to geometric alignment (CMC, FOCAL).

\subsection{Multimodal Pairing Efficiency}
We also evaluate \model's alignment performance at varying amounts of multimodal data for MOD application in Table \ref{tab:multimodal_ratio}.
We align both encoders pretrained from domain M ($M_\text{sei}||M_\text{aco}$) and compare them to standard joint pretraining with different ratios of multimodal data. 
Additionally, we provide supervised performance on the same amount of labeled data used for finetuning.
\model consistently achieves superior multimodal data efficiency, with minimal degradation as we reduce the number of multimodal pairs.
\model has an average of 3.42\% gain over the highest-performing baselines and over 60\% compared to joint model pretraining, which performs poorly in the absence of multimodal data.
Joint pretraining even performs worse than the supervised approach with only 5\% of multimodal data, indicating the standard self-supervised pretraining fails to learn effective representations with an insufficient amount of synchronized multimodal data. In contrast, the two-stage learning paradigm of \model leveraging widely available unsynchronized unimodal data could effectively mitigate this problem.
\begin{figure}[!t]
\centering
\includegraphics[width=0.48\textwidth]{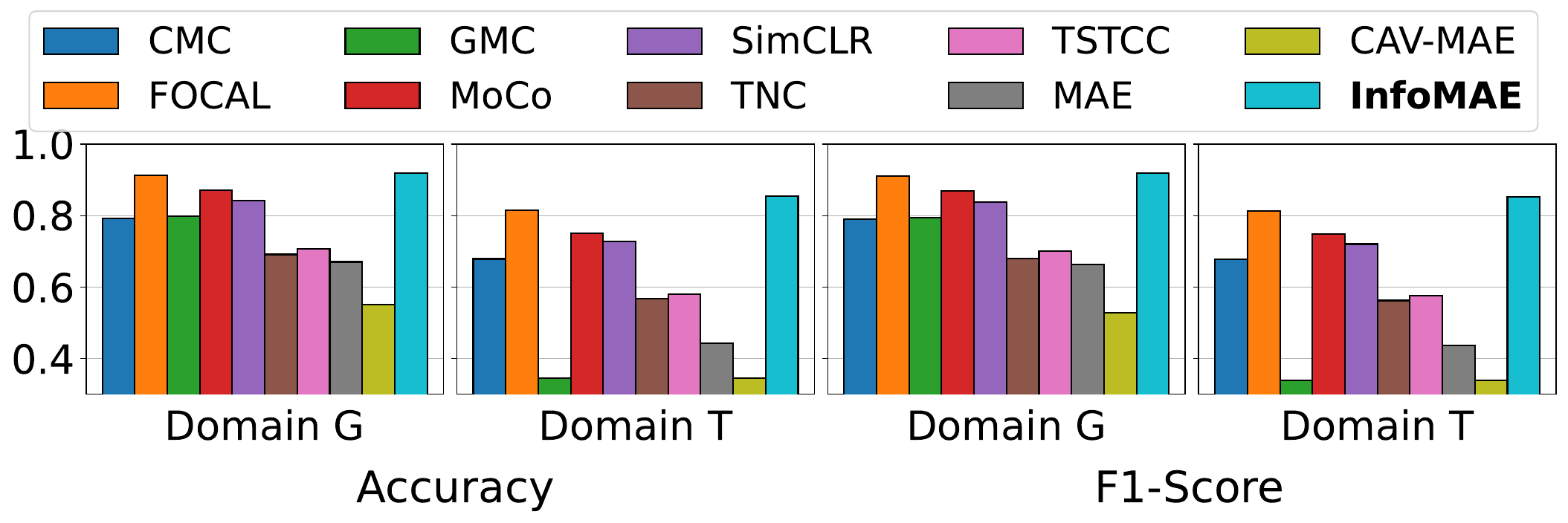}
\vspace{-0.7cm}
\caption{Performance of Joint Pretraining on MOD (seismic and acoustic) dataset and then finetuned on unseen domains.}
\label{fig:domain_generalization}
\end{figure}

\subsection{Standard Mutimodal Pretraining on Large-scale Synchronized Dataset}
While \model excels as an efficient cross-modal alignment framework under limited pairs, it also demonstrates remarkable flexibility as a standard multimodal SSL framework. We evaluate \model against prior state-of-the-art works on Joint Multimodal Pretraining using abundant multimodal pairs, as shown in Figure~\ref{fig:domain_generalization}.
We use synchronized, unlabeled multimodal data from the MOD dataset to pretrain backbone encoders. Then we freeze the pretrained encoders and perform linear probing using labeled multimodal data from domains $G$ and $T$, as described in Section~\ref{subsec:exp_setup}.
\model consistently outperforms the MAE-based framework and achieves better performance than other contrastive baselines.
We leave more evaluation on Joint Multimodal Pretraining across four real-world datasets to Appendix~\ref{appendix:evaluation}.
Prior works, primarily designed for joint multimodal pretraining, often struggle with limited multimodal pairs and show significant performance degradation.
In contrast, \model not only improves multimodal pairing efficiency but maintains high performance with minimal performance degradation.

\subsection{Ablation Study}\label{subsec:ablation}
Finally, we study how each module of InfoMAE contributes to its performance through ablation studies. 
We evaluate four variants of \model by removing temporal, shared, private, and augmentation components in Table \ref{tab:ablation_matching}. 
The absence of either shared or private components leads to a significant degradation, implying the significance of factorized representation for cross-modal alignment.
The drop in performance after removing temporal locality constraints also indicates the importance of learning temporal correspondence for time-series signals.
Without temporal locality, the learned representations lose crucial temporal correspondence and can significantly compromise the ability to learn multimodal correspondences on top of the unimodal representations.
Conversely, \model without augmentations does not significantly reduce the performance, demonstrating its robustness toward augmentation choices, in contrast to many contrastive learning frameworks that require careful selection of augmentations to avoid representational collapses.


\begin{table}[]
\centering
\caption{Ablation accuracy of MOD cross-modal alignment.}
\label{tab:ablation_matching}
\resizebox{0.95\columnwidth}{!}{%
\begin{tabular}{@{}c|c|c|c|c|c@{}}
\toprule
Frameworks & $T_\text{sei}||M_\text{aco}$     & $G_\text{sei}||T_\text{aco}$     & $T_\text{sei}||T_\text{aco}$     & $G_\text{sei}||M_\text{aco}$     & $T_\text{sei}||G_\text{aco}$     \\ \midrule
noTemp     & 0.6946 & 0.5881 & 0.5044 & 0.7435 & 0.6651 \\
noShared   & 0.7683 & 0.6504 & 0.5298 & 0.8125 & 0.7395 \\
noPrivate  & 0.5479 & 0.4180 & 0.2873 & 0.6259 & 0.5399 \\
noAug      & 0.7863 & 0.6973 & 0.5881 & 0.8232 & 0.7924 \\ \midrule
\textbf{InfoMAE} & \textbf{0.7950} & \textbf{0.6986} & \textbf{0.5928} & \textbf{0.8326} & \textbf{0.8326} \\ \bottomrule
\end{tabular}%
}
\end{table}

\section{Related Work}\label{sec:related}

\textbf{Self-Supervised Multimodal Learning.} 
Self-supervised learning (SSL) techniques, such as Contrastive Learning (CL) and masked reconstructions, have achieved significant success in visual, textual, and time-series representation learning \cite{caron2021emerging, eldele2021time, emadeldeen2022catcc, grill2020bootstrap, radford2021learning, shen2021much, tonekaboni2021unsupervised, yang2023simper, yue2022ts2vec, zhang2022tfc}. 
Masked reconstruction learns informative representations by reconstructing masked inputs \cite{brown2020language, devlin2018bert, he2022mae, kong2023understanding, xie2022simmim}, with various masking strategies explored \cite{bandara2023adamae, denizhan2024freqmae, zhang2022mask}, and extended to time-frequency spectrograms \cite{huang2022masked, kara2024phymask} and videos \cite{gupta2024siamese, tong2022videomae}.
Multimodal representation learning has become increasingly important with diverse applications \cite{bos2006eeg, liangmultibench, reiss2012introducing, schmidt2018wesad, zheng2021deep}. 
Recent works leverage CL to learn correspondences between modalities \cite{deldari2022cocoa, pielawski2020comir, ouyang2022cosmo, poklukar2022gmc, tian2020contrastive, trosten2021reconsidering, zolfaghari2021crossclr}, and others pretrain unified encoders for multimodal representations \cite{hu2021unit, mizrahi20244m}.
Factorized Multimodal Learning \cite{liu2024focal, liang2024factorized, jiang2023understanding, tsai2018learning, hsu2018disentangling} further decouples multimodal learning by acknowledging both modality-specific and modality-shared information.
FOCAL \cite{liu2024focal} proposed contrastive learning objectives to learn shared and private representation in the orthogonal space.
FactorizedCL \cite{liang2024factorized} separates the shared and private space based on their relevance to the downstream tasks. 
Some works~\cite{gong2022contrastive, wang2023contrastive} combine CL with MAE to capture cross-modal correspondence.
Yet, these works minimize the geometric modality gap to learn cross-modal correspondences and rely on massive amounts of multimodal data for joint multimodal pretraining. In contrast, \model minimizes the information modality gap to further enhance the downstream performance.
In reducing multimodal data pairs for training, many works~\cite{tran2017missing, ma2021smil, wang2023multi} propose to impute missing modality pairs through feature generations.
Wang \et~\cite{wang2023connecting} proposes using CL to align multimodal encoders through an anchor modality yet still overlooking unimodal data. 
In contrast, \paradigmabbrv minimizes the reliance on multimodal data by taking advantage of a large amount of unimodal data. 

\noindent\textbf{Multimodal Information Theory}.
There has been a long history of exploring common information between random variables in information theory \cite{wyner1975common, gacs1973common, yu2022common}, and it is still an active research field \cite{hanna2022can, hanna2023common,hanna2024on, erixhen2022gaussian}. However, it remains challenging to compute the common information in practical applications. Kleinman1~\et~\cite{kleinman2024gacs} combines Variational Autoencoders with Gacs-Korner Common Information. Mai~\et ~\cite{mai2022multimodal} proposes to measure the information redundancy for multimodal data. However, they do not explicitly consider the unique information for factorization.
\model adopts the informational factorization considering both private and shared information to construct a joint representation in a task-agnostic manner rather than extracting task-related information like \cite{liang2024factorized}.


\section{Discussion \& Conclusion}\label{sec:conclusion}

In this paper, we proposed \paradigm, a pairing-efficient multi-stage SSL paradigm for multimodal IoT sensing. 
It first pretrains independent modality encoders on large-scale unimodal data sets. Then, it leverages a novel information theory-based optimization to achieve distributional cross-modal alignment with only limited multimodal pairs. 
Extensive evaluations compared to standard multimodal SSL frameworks demonstrated the superior efficiency and effectiveness of \model across multiple real-world IoT applications. 
We believe it opens new opportunities for developing more data-efficient and qualitative self-supervised multimodal models.
In the Appendix, we provide additional evaluations and describe more details on the proof, datasets, implementation, and limitations.

\begin{acks}
Research reported in this paper was sponsored in part by the Army Research Laboratory under Cooperative Agreement W911NF-17-20196, NSF CNS 20-38817, and the Boeing Company. 
It was also supported in part by ACE, one of the seven centers
in JUMP 2.0, a Semiconductor Research Corporation (SRC)
program sponsored by DARPA.
The views and conclusions contained in this document are those of the author(s) and should not be interpreted as representing the official policies of the CCDC Army Research Laboratory, or the US government. The US government is authorized to reproduce and distribute reprints for government purposes notwithstanding any copyright notation hereon.
\end{acks}

\bibliographystyle{abbrv}
\bibliography{main}
 \appendix 
\section*{Appendix}
\section{Information Formulation}\label{appendix:formulation}

\subsection{Proof of the Equivalence between \eqref{eq:def-1} and \eqref{eq:def-2}}
We first show the equivalence between the condition \eqref{eq:def-1} and the constraints in \eqref{eq:def-2} by proving the following proposition.
\begin{proposition}\label{prop}
    For random variables $X_1$, $X_2$, if $U=s_1(X_1) = s_2(X_2)$, and there exists $W = g_1(X_1) = g_2(X_2)$ such that $X_1\indep X_2\mid W$, then the following two statements are equivalent.
\begin{small}
\begin{enumerate}[label=(\alph*), leftmargin=15pt]
    \item $(\forall f_1,f_2) \Big([f_1(X_1)=f_2(X_2)]\implies  [(\exists f) f(U) = f_1(X_1) = f_2(X_2)]\Big). \quad$ \label{eq:equiv-1}
    \item There is a one-to-one mapping between $W$ and $U$ (i.e., $X_1\indep X_2\mid U$). \label{eq:equiv-2}
    \vspace{-1.5mm}
\end{enumerate}
\end{small}
\end{proposition}
\begin{proof}
We first prove the direction \ref{eq:equiv-2} $\implies$ \ref{eq:equiv-1} using properties of basic information-theory measures (Chapter~2 in \cite{cover1999elements}). For any $f_1,f_2$ such that $f_1(X_1)=f_2(X_2)$, we have
\vspace{-1.5mm}
{\small
\begin{equation}\label{eq:zero_mi}
        0 \stackrel{(i)}{=} I(X_1; X_2|U) \stackrel{(ii)}{\ge} I(f_1(X_1);f_2(X_2)|U) \stackrel{(iii)}{\ge} 0,
\vspace{-0.5mm}
\end{equation}
}

where $(i)$ follows that $X_1$ and $X_2$ are independent conditioned on $U$; $(ii)$ is due to the data processing inequality of mutual information; and $(iii)$ is because the mutual information is always non-negative. \eqref{eq:zero_mi} implies that $I(f_1(X_1);f_2(X_2)|U)=0$. In addition, since $I(f_1(X_1);f_2(X_2)|U)=H(f_1(X_1)|U)-H(f_1(X_1)|f_2(X_2),U)$ and $H(f_1(X_1)|f_2(X_2),U)=0$, we have $H(f_1(X_1)|U)=0$. This concludes that there exist a deterministic function $f$ such that $f(U)=f_1(X_1)=f_2(X_2)$. 

Next, we prove the other direction \ref{eq:equiv-1} $\implies$ \ref{eq:equiv-2}. Note that $W$ given in the proposition statement satisfies $W=g_1(X_1)=g_2(X_2)$ and therefore, from \ref{eq:equiv-1}, we know that there exist a function $h_1$ such that $W=h_1(U)$. Since $W$ also satisfies that $X_1\indep X_2\mid W$ and $U=s_1(X_1)=s_2(X_2)$, then applying the direction \ref{eq:equiv-2} $\implies$ \ref{eq:equiv-1}, we have that $U=h_2(W)$ for some function $h_2$. Therefore, there is a one-to-one mapping between $W$ and $U$.
\end{proof}

Note that it is difficult to obtain a random variable $U$ that satisfies \ref{eq:equiv-1} (i.e. the sufficient common variable in Defined \ref{definition:shared}). The Proposition \ref{prop} allows us to find a random variable $W$ (if it exists) instead. And the one with minimum entropy can be obtained by solving the optimization problem \eqref{eq:def-2}.

\begin{table*}[!t]
\centering
\caption{Statistical summaries of domains and datasets}
\label{tab:datasets}
\resizebox{0.9\textwidth}{!}{%
\begin{tabular}{@{}c|ccccc|ccc@{}}
\toprule
Dataset &
  Modalities (Freq) &
  Sample Length &
  Overlap &
  Classes &
  \begin{tabular}[c]{@{}c@{}}\#Pretrain\\ Samples\end{tabular} &
  \begin{tabular}[c]{@{}c@{}}Used for\\ Alignment\end{tabular} &
  \begin{tabular}[c]{@{}c@{}}\#Alignment\\ Samples\end{tabular} &
  \begin{tabular}[c]{@{}c@{}}\# Finetune\\ Samples\end{tabular} \\ \midrule
Domain M      & acoustic (8kHz) seismic (100Hz) & 2 sec & 0\%  & 5 sec & 39,609 & \checkmark & 1981 & 734          \\
Domain G      & acoustic (8kHz) seismic (100Hz) & 2 sec & 0\%  & 2 sec & 35,168 & \xmark     & -    & 3136 (joint) \\
Domain T      & acoustic (8kHz) seismic (100Hz) & 2 sec & 0\%  & 4     & 43,819 & \xmark     & -    & 4205 (joint) \\ \midrule
PAMAP2        & acc, gyro, mag, lig (all 50Hz)  & 5 sec & 50\% & 18    & 9,611  & \checkmark & 4805 & 961          \\
RW-HAR & acc, gyr, mag (all 100Hz)       & 2 sec & 50\% & 8     & 12,887 & \xmark     & -    & -            \\ \bottomrule
\end{tabular}%
}
\end{table*}

\subsection{Derivation of the Shared Loss \eqref{eq:loss-1}}
We first group the terms that only depend on $U_1$ or $U_2$ as follows.
{\small
\begin{align}
    \mathcal{L}_\text{info}^\text{shared} &= \alpha d(U_1, U_2) + \beta (H(U_1) + H(U_2)) + I(X_1; X_2 \mid U_1) \\ &+ I(X_1; X_2 \mid U_2)\nonumber \\
     &= \alpha d(U_1, U_2) + \mathcal{L}{(U_1)} + \mathcal{L}{(U_2)}, \label{eq:shared-group}
\end{align}
}
where $d(U_1, U_2)$ can be measured using the Euclidean distance or other distance measures. And
{\small
\begin{align}
        \mathcal{L}{(U_1)} &= I(X_1; X_2 | U_1) + \beta H(U_1) \nonumber\\
        &\stackrel{(i)}{=} I(X_1; X_2 | U_1) + \beta I(X_1; U_1) \nonumber\\
        &\stackrel{(ii)}{=} \mathbb{E}_{U_1} \left[D_{KL}(p_{X_1,X_2|U_1} || p_{X_1|U_1}p_{X_2|U_1})\right] \\
        & \hspace{15pt}+ \beta D_{KL}(p_{X_1,U_1} || p_{X_1}p_{U_1}) \nonumber\\
        &= \mathbb{E}_{X_1,X_2,U_1} \left[ \log \frac{p_{X_1,X_2|U_1}}{p_{X_1|U_1}p_{X_2|U_1}} \right] + \beta \mathbb{E}_{X_1,U_1} \left[ \log \frac{p_{X_1,U_1}}{p_{X_1}p_{U_1}} \right] \nonumber\\
        &= \mathbb{E}_{X_1,X_2,U_1} \left[ \log \frac{p_{X_1,X_2,U_1}p_{U_1}}{p_{X_1,U_1}p_{X_2,U_1}} \right] + \beta \mathbb{E}_{X_1,U_1} \left[ \log \frac{p_{X_1,U_1}}{p_{X_1}p_{U_1}} \right] \nonumber\\
        &= \mathbb{E}_{X_1,X_2,U_1} \left[ \log \frac{p_{X_1,X_2,U_1}}{p_{X_1}p_{X_2}p_{U_1}} + \log\frac{p_{X_1}p_{U_1}}{p_{X_1,U_1}}+\log\frac{p_{X_2}p_{U_1}}{p_{X_2,U_1}}\right] \\
        & \hspace{15pt} + \beta \mathbb{E}_{X_1,U_1} \left[ \log \frac{p_{X_1,U_1}}{p_{X_1}p_{U_1}} \right] \nonumber\\
        &= \mathbb{E}_{X_1,X_2,U_1} \left[ \log \frac{p_{X_1,X_2,U_1}}{p_{X_1}p_{X_2}p_{U_1}} \right. \\
        & \left. \hspace{55pt} + (1 - \beta) \log \frac{p_{X_1,U_1}}{p_{X_1}p_{U_1}} + \log \frac{p_{X_2,U_1}}{p_{X_2}p_{U_1}} \right], \label{eq:shared-group1}
\end{align}
}
where $(i)$ follows the relation between mutual information an entropy that $I(X_1;U_1) = H(U_1) - H(U_1|X_1)$ and $H(U_1|X_1)=0$ because $U_1$ is a deterministic function of $X_1$; $(ii)$ is by definition of the conditional mutual information; and the remaining equalities use the Bayes' rule. Similarly, we have
{\small
\begin{flalign}
        &\mathcal{L}{(U_2)} = I(X_1; X_2 | U_2) + (1 - \beta )H(U_2) &&\nonumber\\
        &= \mathbb{E}_{X_1,X_2,U_2} \left[ \log \frac{p_{X_1,X_2,U_2}}{p_{X_1}p_{X_2}p_{U_2}}  + \log \frac{p_{X_1,U_2}}{p_{X_1}p_{U_2}} + \beta \log \frac{p_{X_2,U_2}}{p_{X_2}p_{U_2}} \right]\label{eq:shared-group2}
\end{flalign}
}
Combining \eqref{eq:shared-group}, \eqref{eq:shared-group1} and \eqref{eq:shared-group2}, we can obtain
\begin{small}
\begin{align}
    \mathcal{L}_\text{info}^\text{shared}  &= \alpha d(U_1, U_2) + \sum_{i=1}^{2} \mathbb{E}_{X_1, X_2, U_i} \left[ \log \frac{p_{X_1, X_2, U_i}}{p_{X_1}{p_{X_2}}{p_{U_i}}} \right.\\
    & \left. \hspace{55pt}+ (1 - \beta) \log{p_{X_i, U_i} \over p_{X_i}p_{U_i}} + \log{p_{X_{3-i}, U_i} \over p_{X_{3-i}}p_{U_i}}\right].
\end{align}
\end{small}

\subsection{Derivation of the Private Loss \eqref{eq:private}}
Similar to \eqref{eq:shared-group1}, since $H(V_1|X_1)=H(V_2|X_2)=0$, we have that
{\small
\begin{equation}
\begin{aligned}
    \mathcal{L}_\text{info}^\text{private}  &= \gamma H(V_{1}) + \gamma H(V_{2}) + \epsilon I(V_1; U_1) + \epsilon I(V_2; U_2), \\
    & = \gamma I(X_1;V_{1}) + \gamma I(X_2;V_{2}) + \epsilon I(V_1; U_1) + \epsilon I(V_2; U_2), \\
     &= \sum_i \mathbb{E}_{X_i, V_i, U_i} \left[ \gamma \log{p_{X_i, V_i} \over p_{X_i}p_{V_i}} + \epsilon \log{p_{V_i, U_i} \over p_{V_i}p_{U_i}} \right].
\end{aligned}
\end{equation}
}
\section{Datasets}\label{appendix:datasets}
This section describes the cross-modal alignment and joint multimodal pretraining datasets from two applications: Moving Object Detection (MOD) and Human Activity Recognition (HAR). Table~\ref{tab:datasets} provides the statistical values of each domain.

\subsection{Cross-modal Alignment Datasets}\label{appendix:alignment_dataset}
\subsubsection{Moving Object Detection}

We have seismic and acoustic signals describing different vehicles on three different domains. For simplicity, we use one letter to represent each domain.

\noindent\textbf{Domain M} is a publicly released~\cite{liu2024focal} moving object detection dataset consisting of signals from 7 different moving vehicles, recorded at three different distances and four different speeds. 

\noindent\textbf{Domain G} contains a self-collected dataset on state park grounds near an outdoor research facility with four sensor nodes deployed.
The dataset contains four distinct targets navigating the neighborhood near the sensors in some arbitrary order.

\noindent\textbf{Domain T} has a similar setup as MOD but involves different targets and scenes. This set contains data collected from a paved parking lot, unpaved trails, and gravel roads within a park. Vibration signals of 2 standard-size SUVs from different manufacturers, one lightweight sports car, and one muscle car were recorded. One hour of data for each vehicle was collected at each scene.  use the first 50 minutes for training and the last 10 minutes for validation and testing. 
\subsubsection{Human Activity Recognition}

Unlike the MOD application, where we used data from different domains for unimodal pretraining, we leveraged two different HAR datasets for unimodal pretraining and cross-modal alignment to evaluate the scenario in which IMU data has high degrees of heterogeneity.

\noindent\textbf{RW-HAR \cite{sztyler2016body}}\label{appendix:dataset_realworld_har} is a public dataset with accelerometer, gyroscope, magnetometer, and light signals sampled at 50Hz. It includes data from 15 subjects performing 8 human activities. We use the data collected from the subjects' waist and randomly select ten subjects for training, 2 for validation, and 3 for testing.

\noindent\textbf{PAMAP2 \cite{reiss2012introducing}}\label{appendix:dataset_pamap2} contains inertial data from 18 human daily activities performed by 9 subjects. PAMAP2 includes 9,611 instances, with data captured using inertial measurement units (IMUs) placed on the chest, the wrist of the dominant arm, and the dominant side’s ankle. We use the data collected from the wrist. The signal is collected at a sampling rate of 100Hz. 7 random subjects are used for training, and 2 subjects for testing.

\noindent\textbf{Combined} is a concatenated dataset of RealWord-HAR and PAMAP2. Since PAMAP2 does not contain any light signals, we drop the light modality and only use the three IMU modalities for evaluation.  



\section{Data Preprocessing}\label{appendix:preprocess}
We partition the time-series data into segments of uniform length.
Each segment is subdivided into intervals with overlaps. We apply the Fourier transform to the signal in each interval to derive its spectral content, thereby retaining both temporal and spectral characteristics.
During training, we adopt the same augmentations as FOCAL~\cite{liu2024focal} to the input before and after the fourier-transform. 

\begin{table}[t!]
\centering
\caption{Inference profiling on Raspberry Pi 4 device.}
\vspace{-0.4cm}
\label{tab:appendix_infer}
\resizebox{0.9\columnwidth}{!}{%
\begin{tabular}{@{}ccccc@{}}
\toprule
App. & P99 (s) & Average (s) & Model Size (MB) & \# Parameters (M) \\ \midrule
MOD          & 0.5803  & 0.2259      & 47.9820         & 12.565831                \\
HAR          & 0.1728  & 0.1690      & 17.8810         & 4.669818                 \\ \bottomrule
\end{tabular}%
}
\end{table}

\section{Experiment and Implementation Details}\label{appendix:training}

During pretraining, we randomly sample a batch of sequences of $L$ consecutive samples.
We jointly optimize the backbone encoders and decoders with AdamW \cite{loshchilov2018decoupled} optimizer and Cosine scheduler \cite{loshchilov2017sgdr}. 
We also train discriminators for density-ratio estimations \cite{sugiyama2012density, kim2018disentangling}. 
We apply convolution blocks to map the time-frequency sample into a one-dimensional embedding to match the input dimension $X_1$ with their shared and private representations $V_1, U_1$, followed by 5-layer MLP to their density ratio.

While InfoMAE’s training requires additional computation due to the discriminators and the MAE architecture, we would like to note that \model incurs no extra inference overhead.
We evaluated InfoMAE’s inference performance on a Raspberry Pi 4 device and present the computational overhead in Table~\ref{tab:appendix_infer}.
The result demonstrates that InfoMAE achieves real-time inference in less than 1 second, making it suitable for real-time deployment in WoT/IoT applications where only low-end devices are available.

\noindent\textbf{Computation}. We conducted our experiments on NVIDIA RTX 4090 GPUs (24GB). The training time varies from minutes for finetuning to 2 days for pretraining. The training time for cross-modal alignment is faster with fewer multimodal pairs. 

\section{Additional Evaluation}\label{appendix:evaluation}
\begin{table}[!t]
\centering
\caption{Cross-modal alignment with sparse pairs.}
\vspace{-0.3cm}
\label{tab:sparse_align}
\resizebox{0.9\columnwidth}{!}{%
\begin{tabular}{@{}c|cc|cc|cc@{}}
\toprule
Framework & \multicolumn{2}{c|}{GMC} & \multicolumn{2}{c|}{FOCAL} & \multicolumn{2}{c}{InfoMAE}       \\ \midrule
Ratio     & Acc         & F1         & Acc          & F1          & Acc             & F1              \\ \midrule
0.01      & 0.8252      & 0.8247     & 0.8573       & 0.8556      & \textbf{0.8794} & \textbf{0.8786} \\
0.02      & 0.8305      & 0.8272     & 0.8580       & 0.8573      & \textbf{0.8821} & \textbf{0.8811} \\
0.03      & 0.8667      & 0.865      & 0.8560       & 0.8529      & \textbf{0.8875} & \textbf{0.8841} \\ \bottomrule
\end{tabular}%
}
\end{table}
\subsection{Evaluation: Sparse cross-modal alignment}\label{appendix:sparse}

We conduct additional experiments to evaluate \model under extremely sparse conditions, reducing the availability of multimodal pairs to as low as 1\%, 2\%, and 3\%. These results, presented in Table~\ref{tab:sparse_align}, highlight that \model continues to outperform top-performing baselines across these extremely constrained scenarios. The findings illustrate InfoMAE’s robustness in aligning representations under sparse multimodal pairing conditions.

\subsection{Joint Multimodal Pretraining}

\begin{figure}[!t]
\centering
\includegraphics[width=\columnwidth]{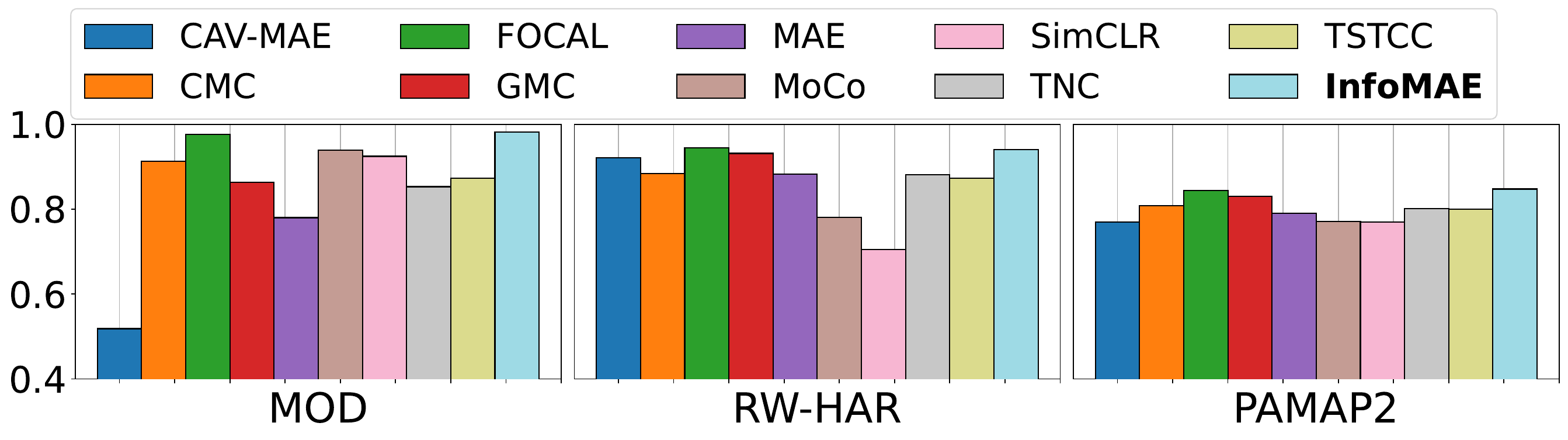}
\vspace{-0.6cm}
\caption{Joint Multimodal Pretraining compared with previous joint pretraining SSL frameworks on four datasets.}
\label{fig:joint_pretrain}
\end{figure}

\begin{table}[!t]
\centering
\caption{Ablation accuracy with Joint Pretraining.}
\vspace{-0.3cm}
\label{tab:ablations}
\resizebox{0.7\columnwidth}{!}{%
\begin{tabular}{@{}c|c|c|c@{}}
\toprule
Frameworks     & MOD                      & RW-HAR   & PAMAP2          \\ \midrule
woTemp         & 0.8734                    & 0.8442          & 0.6948          \\
woShared       & 0.9531                    & 0.8771          & 0.8095          \\
woPrivate      & 0.9082                    & 0.9100          & 0.8080          \\
woAugmentation & 0.9538                    & 0.9106          & 0.8163          \\ \midrule
InfoMAE        & \textbf{0.9826} & \textbf{0.9411} & \textbf{0.8478} \\ \bottomrule
\end{tabular}%
}
\end{table}

Although \model is primarily designed for learning settings where the multimodal pairs are scarce, \model demonstrates strong flexibility and generalization as a standard multimodal SSL framework when abundant multimodal pairs are available. 
Figure~\ref{fig:joint_pretrain} presents additional finetuning performance on joint multimodal pretraining.
\model significantly exceeds the MAE-based SSL framework and achieves comparable or superior performance to the SOTA baselines.
It is noteworthy that these baselines are mainly designed for joint multimodal pretraining. InfoMAE is a universal framework for cross-modal alignment that achieves comparable performance as multimodal SSL with few sacrifices.

\subsection{Additional Ablation Studies}
In Table~\ref{tab:ablations}, we present additional ablation accuracy on joint multimodal pretraining, evaluating variants \model when abundant multimodal data is available.
We find the results consistent with the performance presented in Section~\ref{subsec:ablation}
\section{Limitations and Future Work}\label{appendix:limitations}




\noindent \textbf{Pretraining Overhead and Efficiency.} Compared to contrastive SSL (\eg FOCAL, CMC, etc.), InfoMAE incurs additional computational overhead due to its autoencoder architecture and density ratio estimation. While this enhances multimodal alignment, it increases training complexity.
Future work could explore concurrent unimodal pretraining, optimized attention mechanisms like FlashAttention, and alternative density ratio estimation techniques without training discriminators to improve efficiency.

\noindent \textbf{Potential Bias and Robustness Under Sparse Sampling.} InfoMAE demonstrates resilience under sparse multimodal settings (Appendix~\ref{appendix:sparse}). However, we would like to note that distribution-based alignment cannot completely eliminate sampling biases, which can affect learned representations.
Further research is required to develop more robust alignment methods that mitigate sampling errors and improve generalization under extreme data sparsity.

\end{document}